\documentclass{article}
\usepackage{ijcai21}

\usepackage{times}
\usepackage{soul}
\usepackage{url}
\usepackage[hidelinks]{hyperref}
\usepackage[utf8]{inputenc}
\usepackage[small]{caption}
\usepackage{graphicx}
\usepackage{amsmath}
\usepackage{amsthm}
\usepackage{booktabs}
\usepackage{algorithm}
\usepackage{algorithmic}
\usepackage[T1]{fontenc}
\usepackage{enumitem}
\newtheorem{theorem}{Theorem}
\newtheorem{definition}{Definition}[section]

\title{Broadly Applicable Targeted Data Sample Omission Attacks}

\author{
    Zhi-Hua Zhou
    \affiliations
    Nanjing University
    \emails
    pcchair@ijcai-21.org
}

\iftrue
\author{
Guy Barash$^1$
\and
Eitan Farchi$^2$\and
Sarit Kraus$^1$\And
Onn Shehory$^1$
\affiliations
$^1$Bar Ilan University, Israel\\
$^2$IBM Haifa Research, Israel\\
\emails
Guy.Barash@wdc.com,
Farchi@IBM.com,
sarit@cs.biu.ac.il,
onn.shehory@biu.ac.il
}
\fi

\begin{document}

\maketitle

\begin{abstract}
We introduce a novel clean-label targeted poisoning attack on learning mechanisms. While classical poisoning attacks typically corrupt data via addition, modification and omission, our attack focuses on data omission only. Our attack misclassifies a single, targeted test sample of choice, without manipulating that sample. We demonstrate the effectiveness of omission attacks against a large variety of learners including deep neural networks, SVM and decision trees, using several datasets including MNIST, IMDB and CIFAR. The focus of our attack on data omission only is beneficial as well, as it is simpler to implement and analyze.
We show that, with a low attack budget, our attack's success rate is above $80\%$, and in some cases $100\%$, for white-box learning. It is systematically above the reference benchmark for black-box learning. For both white-box and black-box cases, changes in model accuracy are negligible, regardless of the specific learner and dataset. We also prove theoretically in a simplified agnostic PAC learning framework that, subject to dataset size and distribution, our omission attack succeeds with high probability against any successful simplified agnostic PAC learner. 
\end{abstract}

\section{Introduction}
Machine learning (ML) mechanisms were typically designed without considering security risks, as recognized by a plethora of 
studies \cite{papernot2016limitations}. Attacks on learners were introduced, many of which corrupt training data. The latter, referred to as {\it data poisoning attacks} \cite{chakraborty2018adversarial}, modify training data to manipulate model behavior at test time.
Many poisoning attacks aim to manipulate the model and impact performance and accuracy. Their aim is that some data samples, not necessarily specific ones, be erroneously classified. Targeted poisoning attacks, e.g., \cite{chen2017targeted} ,\cite{shafahi2018poison}, aim to misclassify a specific targeted data sample, and only that 
sample (or sample set), with a negligible change in model accuracy. 

In this study we focus on a clean-label data omission attack---a specific type of data poisoning attack. While data poisoning typically requires that the attacker gain access to data samples to modify them or their labels, clean-label data omission does not require that, as it does not modify data samples (or labels thereof)---it merely omits some samples. 
We refer to this attack as a {\it Targeted Data Omission} (TDO) attack. Evidently, TDO is simpler to implement compared to attacks that require sample manipulation. Recent studies \cite{chakraborty2018adversarial}, \cite{liu2018survey}, \cite{shafahi2018poison} introduce targeted poisoning attacks, however, unlike TDO, those attacks are not
focused on omission. A data omission attack is introduced in \cite{our_old_paper}, however, unlike our study, it does not examine black-box attacks, it examines only binary classification tasks, and it does not evaluate against benchmarks.

Poisoning attacks are often evaluated against a speciﬁc victim learning mechanism. Thus, if an attack proves effective against learner $L_1$, it may be ineffective against learner $L_2$. Defense against such attacks can be implemented by switching between 
learners, such that at least one learner is immune to the attack. In contrast, our attack is not learner-specific;
it succeeds in attacking various learners. In white-box attacks (i.e., the specific learning method is known to the attacker), the success rate of TDO is typically 80-100\%. In black-box attacks (i.e., the specific learning method is unknown to the attacker), as expected, the success rate decreases, yet it is significantly higher than state-of-the-art benchmark success rates \cite{schwarzschild2020just}.
We demonstrate the effectiveness of TDO against multiple, diverse learners, white-box and black-box alike, via extensive experiments, with datasets including MNIST, IMDB and CIFAR. 


\subsection{Contribution}
We study a novel non-intrusive, clean-label targeted data poisoning attack 
focused on data omission.  
We demonstrate attack effectiveness against a broad class of learners, across datasets. 
We examine 
white-box and black-box attacks. We further show that 
our attack surpasses state-of-the-art benchmark performance. 
Our attack targets a specific data sample and successfully misclassifies it, yet its effect on model accuracy is negligible. This is shown 
theoretically in the context of a simplified agnostic PAC learning, and experimentally. 

\subsection{Related work} \label{sec:related}


A data omission attack is a specific type of data poisoning attack. Poisoning attack methods exhibit several attack approaches, usually adding new data or corrupting existing data (or both), as e.g., in 
\cite{rubinstein2009antidote}, \cite{mozaffari2014systematic}. In contrast, our attack involves no data injection or corruption -- it merely entails data omission. Work in \cite{whenMLfail} presents a method of modeling attacks evaluating against ML. 
The metrics it provides examine several capabilities employed by the attack methods of interest, however they do not examine data omission. 

Data poisoning in \cite{xiao2015support} focuses on malicious label changing. Our attack, however, is a clear-label attack, with no change to labels. Some poisoning attacks allow data omission, but not as an exclusive element of the attack. In \cite{ma2019data}, data deletion is one element of its data manipulation attack, but the effects of data deletion are not examined. In contrast, we focus on exclusive data omission as a major attack vehicle.  
Targeted attacks on a single test sample, \cite{szegedy2013intriguing}, \cite{brown2017adversarial}, resemble the idea of TDO. There, the sample itself is manipulated with adversarial noise. TDO does not modify the target sample. Rather, via omission of other samples, the model itself misclassifies the attacked sample.

A targeted clear-label omission attack was presented in \cite{our_old_paper}.That study is possibly the only published omission-only attack. While we adopt a similar attack approach, our study differs from, and significantly improves upon, that study. It examines only white-box attacks and binary classification tasks, while we examine both white-box and black-box attacks with 2, 3 and 10 classes. Our experimental evaluation is much broader: many more classifiers, more datasets, and comparison to state-of-the-art benchmarks \cite{schwarzschild2020just}. In addition to superior results, we support TDO with theoretical foundations. 

Several defense methods against poisoning attacks were developed, e.g., data sanitation \cite{steinhardt2017certified}, \cite{jagielski2018manipulating}. 
Since TDO introduces no new or modified data, it is immune to defense via data sanitation. 
Many poisoning attacks, 
e.g., \cite{biggio2011support} and \cite{nelson2008exploiting}, 
aim to decrease model accuracy. In contrast, TDO targets a specific data sample while keeping the model's accuracy intact for other test samples.

Poisoning attacks mentioned above as well as others, e.g., \cite{huang2011adversarial}, \cite{li2014feature}, \cite{papernot2016transferability} and \cite{jagielski2018manipulating}, each target a specific learner. In contrast, TDO can be applied to various learners. 
Cross-learner attack capabilities are found in \cite{our_old_paper}, though TDO is much broader in learner coverage and success rates. Cross-learner applicability was shown in \cite{shafahi2018poison}, however not for an omission attack.

\section{Preliminaries and problem statement}
\label{preliminaries}
Assume a data space $X$, a class space $C$, a probability measure $P$ over $X$ and a function $f : X \rightarrow C$ that maps elements of $X$ to classes in $C$. Also assume a learner $L$, an attacker $A$ and a
dataset $S= \{x_1, \ldots, x_n \}$ sampled i.i.d. from $X$ whose elements are labeled by $f$. Further assume that $L$ and $A$ can access $S$, however $P$ and $f$ are unknown to them. A distance function $d(x_i, x_j$) measures the distance between $x_i, x_j \in X$.
The goal of $L$ is to learn a function $g: X \rightarrow C$ that best approximates $f$. Once $g$ is learned by $L$, its accuracy is estimated on a fresh validation sample obtaining $Ac(g)$ (by an oracle, as $f$ is unknown to $L$ and $A$). 


Given test sample $\hat{x} \notin S$, $L$ uses $g$ to assign to $\hat{x}$ a specific label $c \in C$. The goal of $A$ is to alter $L$'s prediction to $c' \neq c \in C$. For convenience, without loss of generality, 
let $\{\textit{\mbox{src}},\textit{\mbox{trgt}}\} \subseteq C$ and $c= \textit{\mbox{src}}$. $A$'s attack aims for $L$ to learn $c' = \textit{\mbox{trgt}}$, instead of $\textit{\mbox{src}}$. $A$ has an attack budget $k$ which is the number of data points $A$ may remove from $S$. For this, $A$ can strategically select $S_k \subseteq S$, $|S_k| \leq k$ for removal, to create a new subset $\hat{S}$ such that  $\hat{S} \cap S_k = \emptyset$ and $\hat{S} \cup S_k = S$.

{\bf Problem statement:} Given the settings above, 
devise an attack that omits up to $k$ data points from $S$ to produce $\hat{S}$; when $L$ is trained on $S$, it produces $g$ with accuracy $Ac(g)$; when $L$ is trained on $\hat{S}$, it produces $\hat{g}$, which consistently misclassifies $\hat{x}$, with accuracy $Ac(\hat{g})$; Accuracies are equivalent, i.e., $|Ac(g)-Ac(\hat{g})| < \alpha$.


\section{Theoretical background}
\label{theory}

\subsection{Attack definition and design} 
Given $X$, $C$, $P$, $L$ and $f : X \rightarrow C$ as specified above. For simplicity, we focus on a class space $C=\{-1, 1\}$. A hypotheses set $H$ is a set of such $f$ functions. 
$L$ attempts to learn $g : X \rightarrow C$ that best approximates $f$.  
Once $g$ is chosen by $L$, it is evaluated against $f$ under the probability measure $P$. To measure the accuracy of $g$, a penalty of $1$ is incurred for each $x_i \in X$ such that $g(x_i) \neq f(x_i)$, and $0$ otherwise. The average accuracy is the overall loss of $g$. 
\begin{definition}{Model Loss.}
The model loss $\mbox{\it Ls}(g) = E(I(f,g))) = P[\{ x \in X : g(x) \neq f(x)\}] $ where $I(f(x), g(x)) = 1$ if $f$ and $g$ disagree and $0$ otherwise. 
\end{definition}

Next we define an agnostic Probably Approximately Correct (PAC) learning. In fact, the classic definition \cite{valiant1984theory} is slightly different as it requires the learning to occur on any probability $P$ over $X \times C$. In order to distinguish from the classic definition, we call our learning framework Simplified Agnostic PAC learning.

\theoremstyle{definition}
\begin{definition}{Simplified Agnostic PAC Learnability (SPAC).}
\label{agnostic_PAC}
A hypothesis class H of functions of type $X \rightarrow C$ is simplified agnostic PAC learnable if $\exists L$ a learning algorithm, such that
$\forall \epsilon, \delta \in (0,1)$, 
$\forall P$ over $X$ is a probability measure, and
$\forall f: X \rightarrow C$, when running $L$ on $m \geq m(\epsilon,\delta)$ i.i.d $(x_1, f(x_1)), \ldots, (x_m, f(x_m))$ examples generated by $P$ and labeled by $f$, $L$ returns a function $g: X \rightarrow C$ such that, with a probability of at least $1 - \delta$ (over the choice of the examples), $\mbox{\it Ls}(g) < min_{h \in H} \mbox{\it Ls}(h) + \epsilon$.
\end{definition}

In Definition \ref{agnostic_PAC} the learner meets the learning requirement. I.e., with a probability of at least $1 - \delta$ (over the choice of the examples), $\mbox{\it Ls}(g) < min_{h \in H} \mbox{\it Ls}(h) + \epsilon$, for any probability distribution $P$ and labeling function $f$. An adversary may modify values of $f$ and $P$ to affect learning results. To facilitate data omission, we focus on modifying the probability $P$, effectively ``omitting" parts of $X$.   


Assume continuous probability $P$ chosen by nature. Adversary $A$ attacks point $\hat{x} \in X$ s.t. $f(\hat{x})=-1$. As $P$ is continuous, the probability of sampling the specific point $\hat{x}$ is $0$.  $A$ tries to change the learning in a sphere $B(\hat{x}, r)$ to increase the probability that the value learned at $\hat{x}$ is $1$.  Theorem \ref{setAttackAgnostic} captures the way in which such an attack is performed.

\begin{theorem}
\label{setAttackAgnostic}
For any $\epsilon, \delta$  agnostic SPAC learning algorithm, $L$, over $H$.  Nature chooses a function $f : X \rightarrow C$ and a continuous probability distribution $P$ over $X$. 
\begin{enumerate}
\item
 Let $\gamma > \epsilon + min_{h \in H} \mbox{\it Ls}(h)$.  Here $min_{h \in H} \mbox{\it Ls}(h)$ is taken over $P$. The adversary $A$ chooses $\hat{x} \in X$ and $r > 0$ so that $B(\hat{x}, r) = \widetilde{X} \subset X$ and $P(\{x \in \widetilde{X} | f(x) = 1\}) > \gamma $.   Next, $A$ can change the probability distribution to $P'$ so that when the $L$ learns using a sample from $P'$, the probability under $P$ or $P'$ of labeling points as $1$ on $\widetilde{X}$ by $L$ will be at least $\gamma - (\epsilon + min_{h \in H} \mbox{\it Ls}(h))$. 
\item
Assume $P[x \in \widetilde{X} \land f(x) = -1 ] < \gamma^{'}$ then the learner error, $P[g(x) \neq f(x) ]$, is bounded above by $ \epsilon + min_{h \in H} \mbox{\it Ls}(h) + \gamma^{'}$ when the learner learns using a sample from $P'$. 
\end{enumerate}
Both claims above will occur with probability
$1-\delta$ when learning from a sample of at least $m(\epsilon, \delta)$ points from $X$ using $P'$ and labeling them using $f$.  
\end{theorem}

\begin{proof}
The adversary first defines the probability $P'$ by
redistributing the probability mass on $\{x \in \widetilde{X} | f(x) = -1\}$ to $\{x \in \widetilde{X} | f(x) = 1\}$, thus obtaining a new probability distribution $P'$ for which $P'(\{x \in \widetilde{X} | f(x) = -1\}) = 0$.  Obviously, $P'(\widetilde{X}) = P(\widetilde{X}) > \gamma$.

As we transferred the probability mass from $\{x \in \widetilde{X} | f(x) = -1\}$ to $\{x \in \widetilde{X} | f(x) = 1\}$ to obtain $P'$, we have that $P'(E) \geq P(E)$ for any event $E \subseteq \{x \in \widetilde{X} \land f(x) = 1 \}$.  Also, for any event $E \subseteq \{x \in X |  \in X - \widetilde{X}\}$ we have that $P'(E) = P(E)$ as the probability mass on $X - \widetilde{X}$ was not modified.  
Next we sketch the proof of parts one and two of the theorem.  The detailed proof is in Appendix A.
\vspace{-3pt}
\begin{enumerate}[leftmargin=15pt]
\item
With a probability of at least $1-\delta$, 
$\epsilon + min_{h \in H} \mbox{\it Ls}(h) > \mbox{\it Ls}(g) = P^{'}[g(x) \neq f(x)]$.  The probability of the error incurred by $P^{'}$ was not changed on $X-\widetilde{X}$.  On $\widetilde{X}$ the probability of the error incurred by $P'$ is $0$ when $g(x) \neq f(x)$ and $f(x) = -1$. We thus have that $P'[g(x) \neq f(x) \land x \in \widetilde{X} \land f(x) = 1]  < \epsilon + min_{h \in H} \mbox{\it Ls}(h)$.  The same applies for $P$ as 
the error incurred by $P'$ on $g(x) \neq f(x)$ and $f(x) = 1$ is greater than the error incurred by $P$ on the same set.  Given that $P(\{x \in \widetilde{X} | f(x) = 1\}) > \gamma $, we get the desired result. 
\item
As before, with probability $1-\delta$, $\epsilon + min_{h \in H} \mbox{\it Ls}(h) > \mbox{\it Ls}(g) = P'[g(x) \neq f(x)]$.  As the errors incurred on $X-\widetilde{X}$ by $P$ and $P'$ are the same and the error incurred by $P'$ is greater than the error incurred by $P$ on $g(x) \neq f(x) \land f(x) = 1 \land x \in  \widetilde{X}$, we have that $P[g(x) \neq f(x)] \leq \epsilon + min_{h \in H} \mbox{\it Ls}(h) +  P[g(x) \neq f(x) \land f(x) = -1 \land x \in \widetilde{X} ] \leq \epsilon + min_{h \in H} \mbox{\it Ls}(h) + \gamma^{'} $.
\end{enumerate}
\end{proof}

\vspace{-0.3in}
\subsection{Implementing the attack}

Bridging between theory and practice is detailed below. 
\begin{itemize}[leftmargin=10pt]
    \item The attacker $A$ of Theorem \ref{setAttackAgnostic} chooses $\widetilde{X} \subset X$ and modifies the probability distribution on $\widetilde{X}$. In practice, $A$ does not have that control. Instead, once a sample $S$ is chosen for learning, $A$ can intervene and omit points in $S$ to obtain an attack sample $S{'}$.  To 
    implement a practical attack, $A$ may omit points from $S$ so that the empirical distribution observed by $L$ on a chosen $\widetilde{X} \subset X$ is modified as in Theorem \ref{setAttackAgnostic}.  Given that the empirical distribution for enough data represents the actual distribution $P$, it is ``as if" $A$ implements the attack described in Theorem \ref{setAttackAgnostic}.
    \item Any PAC learning algorithm is a SPAC learning algorithm.  A learning algorithm with a finite hypothesis space has a finite VC dimension and is thus PAC learnable \cite{shalev2014understanding}, Section 4.2.  Learners attacked in this work can be thought of as having a hypothesis set composed of their parameters, e.g., the weights on the edges of a neural network.  This includes KNN as explained in \cite{KLESK20111882}.  In addition, the learning algorithms attacked in this work can be considered to have a finite hypothesis set by the "Discretization Trick" (see remark 4.2 in \cite{shalev2014understanding}) and are thus PAC learnable.  Essentially, the trick utilizes the fact that the learning model parameters have finite possible values when mapped to a specific machine.    
    \item To apply Theorem \ref{setAttackAgnostic} we choose 
    $\gamma > \epsilon + min_{h \in H} \mbox{\it Ls}(h)$. $\epsilon$ bounds the learner's expected error and can be deduced from its performance. $min_{h \in H} \mbox{\it Ls}(h)$ may be difficult to compute. Yet, if the learning approximates the learned function well, $min_{h \in H} \mbox{\it Ls}(h)$ is expected to be negligible.        
\end{itemize}

We proceed to design the attack (based on Theorem \ref{setAttackAgnostic}).  

%


\begin{enumerate}
\item
Nature chooses a continuous probability distribution $P$ over $X$ and $f: X \rightarrow C$. $P$ and $f$ are unknown to $A$ and $L$.  
\item
Nature samples $S = \{x_1, \ldots, x_m\}$ from $X$ using $P$
and labels them using $f$.
\item
\label{attack}
The attacker $A$ chooses $S^{'} \subset S$.  $S^{'}$ is chosen to be the points in $S$ contained in a sphere $B(\overline{x}, r)$. $A$ estimates the probability $P(z \in B(\overline{x}, r) \land f(z) = 1)$ as $\frac{|S'| - |z \in S' \land f(z) = -1|}{|S|}$ using the empirical distribution deduced from $S$. This is done to ensure that $P(z \in B(\overline{x}, r) \land f(z) = 1) > \gamma$, as required by Theorem \ref{setAttackAgnostic}. 
\item
Next, $A$ drops points in $S^{'}$ that do not have the desired function value, $1$. Thus, $A$ obtains a modified set $S^{''} \subseteq S^{'}$. $A$ then creates a new modified sample set $S_t = (S - S^{'}) \cup S^{''}$ which is the empirical representation of distribution $P'$ of Theorem  \ref{setAttackAgnostic}.
The empirical distribution estimate of $P'(z \in B(\overline{x}, r) \land f(z) = 1)$ now becomes 
$\frac{|S'| - |z \in S' \land f(z) = -1|}{|S| - |z \in S' \land f(z) = -1|}$ which is still larger than $\frac{|S'| - |z \in S' \land f(z) = -1|}{|S|} > \gamma$.  Thus, the requirements of Theorem \ref{setAttackAgnostic} are still met.   
\item
Learner $L$ learns a model using $S_t$ instead of $S$. 
\end{enumerate}


Following this, we restrict the number of omitted points to a given budget $k$, which proves effective experimentally. 

\section{Attack methods}
\label{attack_methods}
Multiple methods for choosing the right samples to alter can be used to carry out an attack. We introduce and study three. 
The inputs to attack methods are a budget $k$, a target data sample $\hat{x}$, a class $\mbox{trgt}$, a dataset $S$ and a learner $L$ or its surrogate.  

\label{assumptions}

\label{attack:KNN}
The \textit{KNN} (K-nearest-neighbours) attack method is inspired by the KNN classification algorithm \cite{altman1992introduction}. 
Given $S$, $k$ and $\hat{x}$, 
we choose a distance measurement appropriate to $S$, such as Euclidean or Cosine distance, and calculate the distance between each point $s_i, f(s_i) \neq \textit{\mbox{trgt}}$  and $\hat{x}$. We choose the $k$ points that are closest to $\hat{x}$ with $f(s_i) \neq \textit{\mbox{trgt}}$ as subset $S_k$ to be removed from $S$, thus creating the attacked dataset $\hat{S} \doteq S \setminus S_k$ (see  Algorithm \ref{algorithm1}).

\begin{algorithm}[htb]
\caption{The KNN attack method}
\label{algorithm1}
\begin{algorithmic}[1] 
\STATE $S_k$ = \{\}   $S_{original} = S$\;
\FOR{j \textbf{in} (1,..,$k$)} 
        \STATE $\#$ Check only points not labeled \textit{\mbox{trgt}} \\
        \STATE $\#$ Choose the point $s_j$ closest to target sample \\
        \STATE $\#$ Remove it from dataset $S$ and add to $S_k$\\
        
\ENDFOR
   
\STATE $\hat{S} = S_{original} \setminus S_k$ \;
\STATE Return $\hat{S}$\;
\end{algorithmic}
\end{algorithm}

\vspace{-6pt}
The \textit{greedy} attack method uses a greedy strategy to determine $S_k$. Given $S$, $k$, $\hat{x}$, a victim $L$ or surrogate of $L$, the algorithm populates $S_k$ with $k$ points, computed as follows. We iterate over all of the points $s_i \in S$. In each $i$, we generate $S_i = S \setminus \{s_i\}$, 
compute a model using $S_i$, and use it to compute the probability that the label of $\hat{x}$ is $\textit{\mbox{src}}$. In both algorithms \ref{algorithm2} and \ref{algorithm3} we denote the probability of learner $L$, given dataset $S$ to classify target sample $\hat{x}$ as $\textit{\mbox{src}}$ as $Pr(L,S,src,\hat{x})$. Note that $L$ could be the victim or its surrogate.  We choose the point $s_j$ for which the lowest probability was derived, remove it from $S$, and add it to $S_k$. This is repeated $k$ times. 
 Algorithm \ref{algorithm2} describes the process.

\begin{algorithm}[htb]
\caption{The Greedy attack method}
\label{algorithm2}
\begin{algorithmic}[1] 
\STATE $S_k$ = \{\}, $S_{original} = S$\;
\FOR{j \textbf{in} (1,..,$k$)} 

        \STATE $\#$ Greedily choose the optimal next point \\
        \STATE $s_j = argmin_{s_i \in{S}} \big(Pr(L, S \setminus \{s_i\},\textit{\mbox{src}},\hat{x}) \big)$\\

        \STATE $\#$ Remove it from dataset $S$ and add to $S_k$\\
\ENDFOR
\STATE $\hat{S} = S_{original} \setminus S_k$ \;
\STATE Return $\hat{S}$\;

\end{algorithmic}
\end{algorithm}
  
\vspace{-6pt}
\label{genetic_how}
The \textit{genetic} attack method uses a genetic
algorithm (GA) to find $\hat{S} \doteq S \setminus S_k$ that optimizes the attack. As common in GAs, we preset $GEN$---the number of generations, and $OS$---the number of offspring per generation. This enables convergence. 
In our GA, an offspring is a set $\textit{\mbox{Off}} \subset S$ of size $k$, generated via genetic computation from parent sets. 
The fitness of an offspring $\textit{\mbox{Off}}$ is evaluated via a fitness function $F(\textit{\mbox{Off}})$ (Eq. \ref{fitness_function}). The two offsprings with the highest fitness value are selected for proceeding generations. Eventually, the most fit offspring is selected as $S_k$. 
\begin{equation}
\label{fitness_function}
F(\textit{\mbox{Off}}) = \mbox{\it Pr}(L, \textit{\mbox{Off}},\mbox{\it trgt},\hat{x}) \in [0,1]
\end{equation}

We initialize the GA by creating two $k$-sized randomly generated subsets of $S$, $\textit{\mbox{Off}}_1 \subset S$ and $\textit{\mbox{Off}}_2 \subset S$. 
These are the ``parents" of the first round of the GA. In each generation we first create the offsprings by randomly picking $k$ points from $\textit{\mbox{Off}}_1 \cup \textit{\mbox{Off}}_2$. To avoid local minima, each offspring undergoes mutation in which each point has a probability $\frac{1}{k}$ of being replaced by a randomly selected point from $S$. In each generation we measure the fitness of each offspring and each parent (using $F()$). 
This is described in Algorithm \ref{algorithm3}. 

\begin{algorithm}[tb]
\caption{The Genetic attack method} 
\label{algorithm3}
\begin{algorithmic}[1] 
\STATE $\#$ Initial\ parents \\
\STATE $\textit{\mbox{Off}}_1, \textit{\mbox{Off}}_2 = \textit{choose-k-samples-from S}$ \; 
\FOR{$gen \in (1,\ldots, GEN$)}
    \FOR{$t \in (3,\ldots, OS + 2)$}
        \STATE $\#$ Create\ offsprings \\
        \STATE $\textit{\mbox{Off}}_t = \textit{choose-k-samples-from-(}\textit{\mbox{Off}}_1 \cup  \textit{\mbox{Off}}_2)$ \;
        \STATE $\#$ Mutate\ offsprings \\
        \STATE $\#$ $Select\ c\in \{0,\ldots,k\} from\ some\ distribution$ \\
        \STATE $\#$ $Randomly\ remove\ c\ samples\ from\ \mbox{Off}_t $\\
        \STATE $\mbox{Off}_t = \mbox{Off}_t \cup (\textit{choose-c-samples-from S})$ \;
    \ENDFOR
    \STATE $\#$ Choose\ top-2\ runners\ as\ new\ parents\ $\textit{\mbox{Off}}_1, \textit{\mbox{Off}}_2$ \\
\ENDFOR
\STATE $\#$ Pick the current top runner \\
\STATE $S_k = \mbox{Off}_1$;
\STATE $\hat{S} = S \setminus S_k$ \;
\STATE Return $\hat{S}$;

\end{algorithmic}
\end{algorithm}

\section{Experiments}
\label{section_experiments}
We evaluated the TDO attack experimentally with 4 different datasets: MNIST, IMDB, CIFAR-10 and a synthetic dataset, denoted {\em Synthetic}. We conducted 4 sets of experiments, one set per dataset.
In each set, we examine multiple attacker-victim combinations. 
We repeated each experiment multiple times to validate and increase confidence.
Each set comprises multiple instances. An instance $i$ refers to a specific dataset $S_i$, victim learner $L_i$, attack method $A_i$, attack budget $k_i$, and target data-point to be attacked ${\hat{x}}^i$. In all experiments we limit the budget to $k_i \leq \sqrt{|S_i|}$. 
In each instance $i$, $L_i$ learns a model ahead of the attack and predicts the class of ${\hat{x}}^i$. Then $A_i$ attacks $S_i$ to create dataset $\hat{S}_i \subset S_i$. Next $L_i$ learns a model using $\hat{S}_i$ and predicts the class of ${\hat{x}}^i$ to examine the success of the attack and measure the drop in accuracy. 

Our experiments comprise both black-box and white-box attacks. As in \cite{schwarzschild2020just}, cases where the attacker knows the exact victim configuration and uses it as a surrogate for the attack are considered white-box attacks. Cases where such knowledge is not available are considered black-box attacks. Yet, in black-box attacks the attacker may know the general type of the victim. 
In black-box attacks in our experiments, we intentionally use a different setting of the classifier. When the KNN attack is used, a surrogate is not needed, unless the victim is a deep network. 

We attack 9 victim classifiers: ANN, Decision Tree (DTree), KNN, Gaussian Naive Bayes (GNB), SVM, 1D convolutional Neural network (1DconvNet), MobileNet-V2, ResNet-18 and VGG11. 
%
Configurations of victim classifiers (in bold), and their corresponding surrogate configuration in black-box attacks are: 
\begin{enumerate}
\item SVM: \textbf{linear kernel} $\rightarrow$ polynomial kernel
\item DTree: \textbf{\textit{gini} criterion} $\rightarrow$ \textit{entropy} criterion
\item KNN: \textbf{K=5} $\rightarrow$ K=3
\item Naive Bayes: \textbf{Gaussian} $\rightarrow$ Multinomial.
\item ANN, 2 hidden layers, neurons per layer: \textbf{16} $\rightarrow$ 8
\item \textbf{MobileNetV2, VGG11, ResNet18} $\rightarrow$ Googlenet

\end{enumerate}


\subsection{Experiment {\em Set I}: attack on {\em Synthetic}}
\label{dataset_I}
In experiment {\em Set I}, $S =$ {\em Synthetic}, similar to the dataset used in \cite{biggio2011support}. 
%
{\em Synthetic} comprises $N= 400$ data samples in ${R}^2$, generated from a 2-dimensional normal distribution with two center-points $c_{\textit{\mbox{src}}}$ and  $c_{\textit{\mbox{trgt}}}$. The samples are labeled either $\mbox{\it src}$ or $\mbox{\it trgt}$.
A successful attack should result in $L$ misclassifying $\hat{x}$ as $\mbox{\it trgt}$ instead of $\mbox{\it src}$. 
For each $\{A, L\}$ pair, $50$ dataset instances $\{S_1, \ldots, S_{50}\}$ are created. 
For each $S_i$, $A$ attempts to cause $L$ to misclassify $\hat{x}$. The success rate of $A$ is averaged across 50 instances. This is performed for both black-box and white-box scenarios. Results appear in Tables \ref{tab:Result_WhiteBox} and \ref{tab:Result_BlackBox}.
The goal of experiment {\em Set I} is to demonstrate the success of TDO on a rather simple dataset, and to examine the capabilities and limitations of the attack algorithms.


\subsection{Experiment {\em Set II}: attack on MNIST}
\label{dataset_II}
In experiment {\em Set II}, $S =$ MNIST, a widely-used dataset that contains images of handwritten digits. Each image is $28\times 28$ pixels. MNIST image labels are in  $\{0,1, \ldots , ,9 \}$, where the label of an image is the corresponding digit. We conduct two experiments on MNIST. 
The first evaluates the success rate of each attack against different types of learners. The second examines attack performance across MNIST classes.

In the first experiment, for each $\{A, L\}$ pair, we examine the success of $A$ in causing $L$ to misclassify $\hat{x}$, whose initial label is $\textit{\mbox{src}}$ 
as $\textit{\mbox{trgt}}$. 
We repeat 50 times per and average the results. 
 In each instance, we randomly select labels $\langle \textit{\mbox{src}}, \textit{\mbox{trgt}} \rangle$ such that $\textit{\mbox{src}} \neq \textit{\mbox{trgt}}$. $S$ is generated by randomly selecting 200 samples with each label, i.e., 400 in total. Using $S$, a model is trained. Then, $\hat{x}$ is chosen with the label $\textit{\mbox{src}}$ such that the model classifies it correctly. Then, attack $A$ is applied. Accordingly, $k=20$ data points are removed from $S$ to generate $\hat{S}$. Next, the model is re-trained on $\hat{S}$. We measure $A$'s success in causing the model to missclasify $\hat{x}$ as $\textit{\mbox{trgt}}$. 

The second experiment examines changes in attack performance across MNIST classes. 
Here, $A = genetic$, as it proved most effective in experiment {\em Set I}. 
For each experiment instance, 2 different MNIST labels are chosen as $\textit{\mbox{src}}$ and $\textit{\mbox{trgt}}$. The instance executes as in the first MNIST experiment. 
%
We examine all possible $\langle \textit{\mbox{src}},\ \textit{\mbox{trgt}} \rangle$ class pairs (90 altogether), with 50 repeats for each pair. 
We measure the average success rate for each pair (see results in Figure \ref{fig:mnist_heat}).

\subsection{Experiment {\em Set III}: attack on IMDB}
\label{dataset_III}
In experiment {\em Set III}, $S=$ IMDB, 
a complex real-life dataset with 50K samples. 
Several state-of-the-art models \cite{camacho2017role} can correctly predict IMDB sample labels with high accuracy. 
In IMDB, samples are English text describing movies. Labels, "1" and "0", represent positive and negative sentiment, respectively.

In {\em Set III}, $L$ is a DNN comprised of a word embedding layer, a 1D convolutional layer followed by a max-pooling layer and 2 dense layers. $L$ learns a model to predict the class of test samples. 
Next, $\hat{x}$ is chosen randomly such that it is labeled "1", and predicted as such by the model. 
$L$ has achieved 80\% accuracy on IMDB, comparable to the state of the art. 
Since accuracies of other learners were at most $70\%$, they were dropped from this experiment (results in Table \ref{tab:Result_WhiteBox}). 

\subsection{Experiment {\em Set IV}: attack on CIFAR10}
\label{dataset_IV}
In experiment {\em Set IV}, $S=$ CIFAR10 and the victims are MobileNetV2, VGG11, and ResNet18, which are pre-trained deep networks taken from \cite{torchvision}. We perform black-box attacks only. The surrogate network is Googlenet. 
We examine TDO performance and compare it to an established benchmark
provided in \cite{schwarzschild2020just}. For this, the settings of our experiment are as in the benchmark, with focus on its \textit{Transfer learning} part.  
%
During the experiment, the victim learns the 10 classes of CIFAR-10 with 250 images from each class. The target sample $\hat{x}$ is a randomly selected CIFAR-10 image. 
Since this is a black-box attack, the attacker does not know the victim and uses Googlenet as a surrogate for feature extraction, producing a vector representation for the images. Using cosine-similarity as a distance function, the KNN attack method is applied. 

\section{Results}
\label{results}
\label{results:compare_section}
In this section we illustrate the effects of the TDO attack, we present its success rate across datasets and victims, and compare it to benchmark results. 
Here, the success rate of a pair $\{L, A\}$ is the percentage of the experiments in which $A$'s attack on $L$ resulted in missclassification of the target sample.

\begin{figure} [ht]
  \centering
  \includegraphics[height=6CM,width=8.9CM]{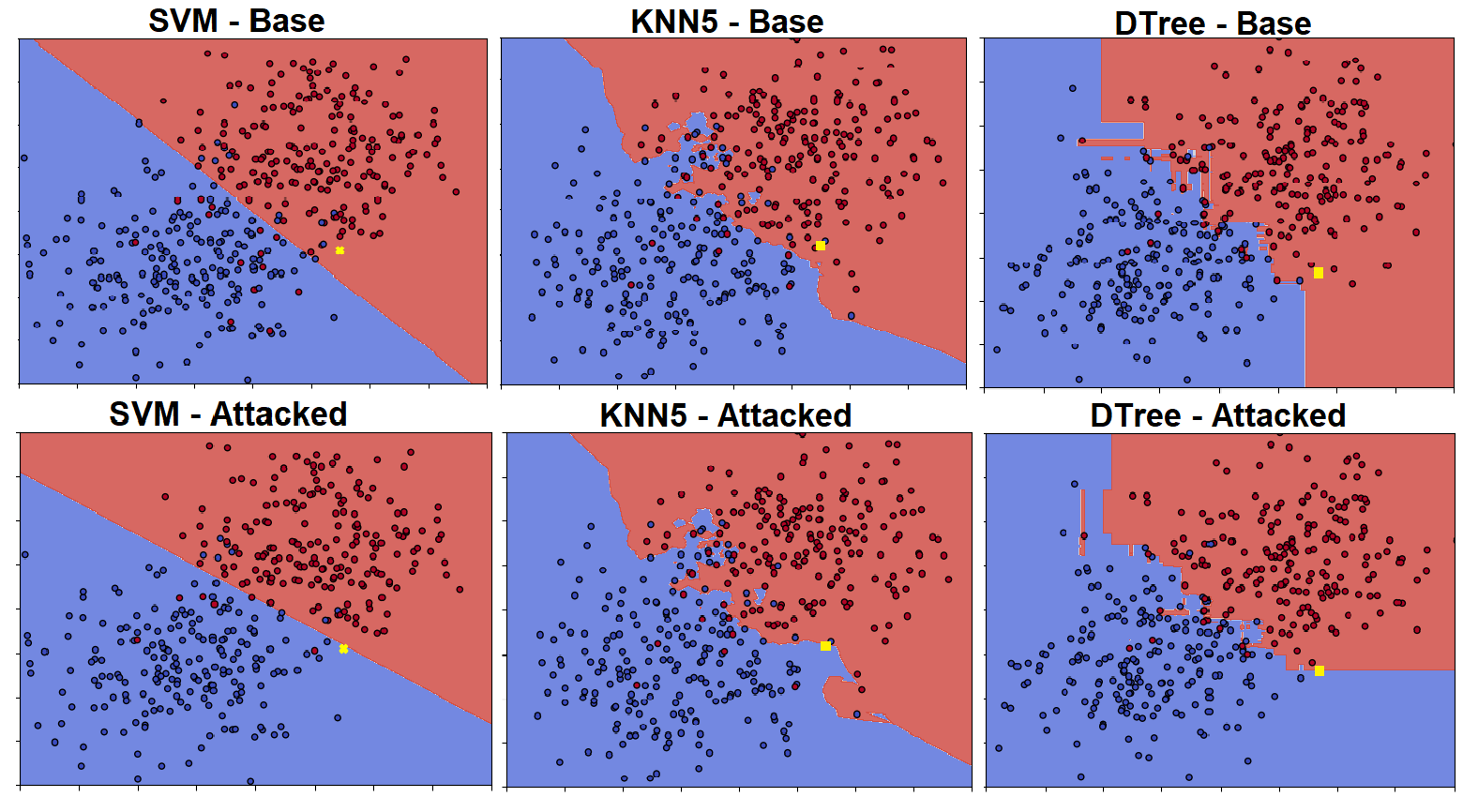}
  \caption{Decision line before and after attack against SVM, KNN5 and DTree. TDO succeeds: the target $\hat{x}$ (yellow) was misclassified.}
  \label{fig:all_learner_before_and_after}
\vspace{-12pt}
\end{figure}

Figure \ref{fig:all_learner_before_and_after} presents examples of results of TDO attacks against SVM, KNN5 and DTree with {\em S= Synthetic}. The border line between blue and red indicates the model's decision line (representing decisions $c = \textit{\mbox{trgt}}$, $c = \textit{\mbox{src}}$, respectively). 
Images at the top 
show the base decision line, prior to the attack. Images at the bottom show the post-attack decision line. The yellow point is the target sample point, $\hat{x}$. 
Observe that, as a result of the attack, the decision line has shifted enough to alter the prediction of $\hat{x}$. The change in model accuracy here and across all TDO experiments was negligibly small. 

Table \ref{tab:Result_WhiteBox} presents success rates of TDO white-box attacks ($Genetic$ and $Greedy$) in experiment sets I-III. White-box attacks were performed against 6 victim learners with datasets IMDB, MNIST and {\em Synthetic}. For brevity, Table \ref{tab:Result_WhiteBox} presents the leading results. Additional results are in Appendix B. One can observe that
the success rates of $Genetic$ are very high and dominate (see other results in the Appendix). Since white-box settings provide access to the victim's configuration, it allows $Genetic$ to adjust well to the specific model, thus delivering very high attack success rates. Observing Table \ref{tab:Result_WhiteBox}, it may be suggested that the GNB learner defend against $Genetic$. However this holds only for {\em Synthetic}, as when GNB is used with MNIST, $Genetic$ succeeds in all attacks.

\begin{table}[ht]
\small
\centering
\begin{tabular}{lllll}
\hline
Dataset & Classes & Victim & Method & Result	\\
\hline
IMDB & 2 & 1DConvNet & Genetic & 0.80	\\
MNIST & 2 & ANN & Genetic & 1.00	\\
MNIST & 3 & ANN & Genetic & 1.00	\\
Synthetic & 2 & KNN5 & Genetic & 0.99	\\
Synthetic & 2 & ANN & Genetic & 0.88	\\
Synthetic & 2 & SVM & Genetic & 0.87	\\
Synthetic & 2 & ANN & Greedy & 0.86	\\
Synthetic & 2 & Dtree & Genetic & 0.85	\\
Synthetic & 2 & GNB & Genetic & 0.58	\\
\hline
\end{tabular}
\normalsize
\caption{Success rates in white-box attacks}
\label{tab:Result_WhiteBox}
\end{table}

\vspace{-3pt}
Table \ref{tab:Result_BlackBox} presents success rates of TDO black-box attacks in experiment sets I, II and IV. Black-box attacks were performed against 8 victim learners with datasets CIFAR-10, MNIST and {\em Synthetic}. Table \ref{tab:Result_BlackBox} presents the leading results. Additional results are in Appendix B. Here, KNN is the dominant attack method. Not surprisingly, success rates are generally lower than in white-box attacks, however they are still impressive across victims.  We compare our CIFAR-10 results to benchmark results reported in Table 2 in \cite{schwarzschild2020just}. Victims MobileNetV2 and VGG11 (lines 1, 2 in  our Table \ref{tab:Result_BlackBox}) are the same victims as in the benchmark. There, the averaged success rate of the two is 0.085. We arrived at success rates of 0.15 and 0.14, respectively, averaged to 0.145, which outperforms the reference results. The ResNet result of TDO is even more impressive. In the benchmark it was suggested that an attack on ResNet is the most difficult and the success rate is poor. In contrast, TDO delivers an impressive success rate of 0.25. Another advantage of our attack, compared to the benchmark, is that the computational resources required by KNN are significantly lower than those required by the attack methods reported there. 

\begin{table}[h]
\small
\centering
\begin{tabular}{llllll}
\hline
Dataset & Clss & Surrogate & Victim & Method & Result    \\
\hline
CIFAR10 & 10 & Googlenet & MoblNetV2 & KNN & 0.15     \\
CIFAR10 & 10 & Googlenet & VGG11 & KNN & 0.14     \\
CIFAR10 & 10 & Googlenet & Resnet18 & KNN & 0.25     \\
MNIST & 2 & X & GNB & KNN & 0.80     \\
MNIST & 2 & X & ANN & KNN & 0.45     \\
MNIST & 2 & ANN & ANN & Genetic & 0.27     \\
MNIST & 2 & SVM & SVM & Genetic & 0.15     \\
MNIST & 3 & X & ANN & KNN & 0.69     \\
MNIST & 3 & GNB & Dtree & Genetic & 0.44     \\
Synthetic & 2 & X & Dtree & KNN & 0.90     \\
Synthetic & 2 & X & ANN & KNN & 0.65     \\
Synthetic & 2 & X & SVM & KNN & 0.48     \\
Synthetic & 2 & X & GNB & KNN & 0.17     \\
\hline
\end{tabular}
\normalsize
\caption{Success rates in black-box attacks}
\label{tab:Result_BlackBox}
\end{table}

The second MNIST experiment examines white-box attack performance across MNIST classes. As observed in Table \ref{tab:Result_WhiteBox}, when the victim is ANN, the success rate is 1.0. When the victim is SVM, the success rate is 0.82. However, success 
varies across classes. Figure \ref{fig:mnist_heat} visualizes this variation. Further to this, we demonstrate success in attacking multi-class classification tasks (2, 3 and 10 classes, shown in the Tables).


\vspace{-10pt}
\begin{figure} [ht]
  \centering
  \includegraphics[height=4CM]{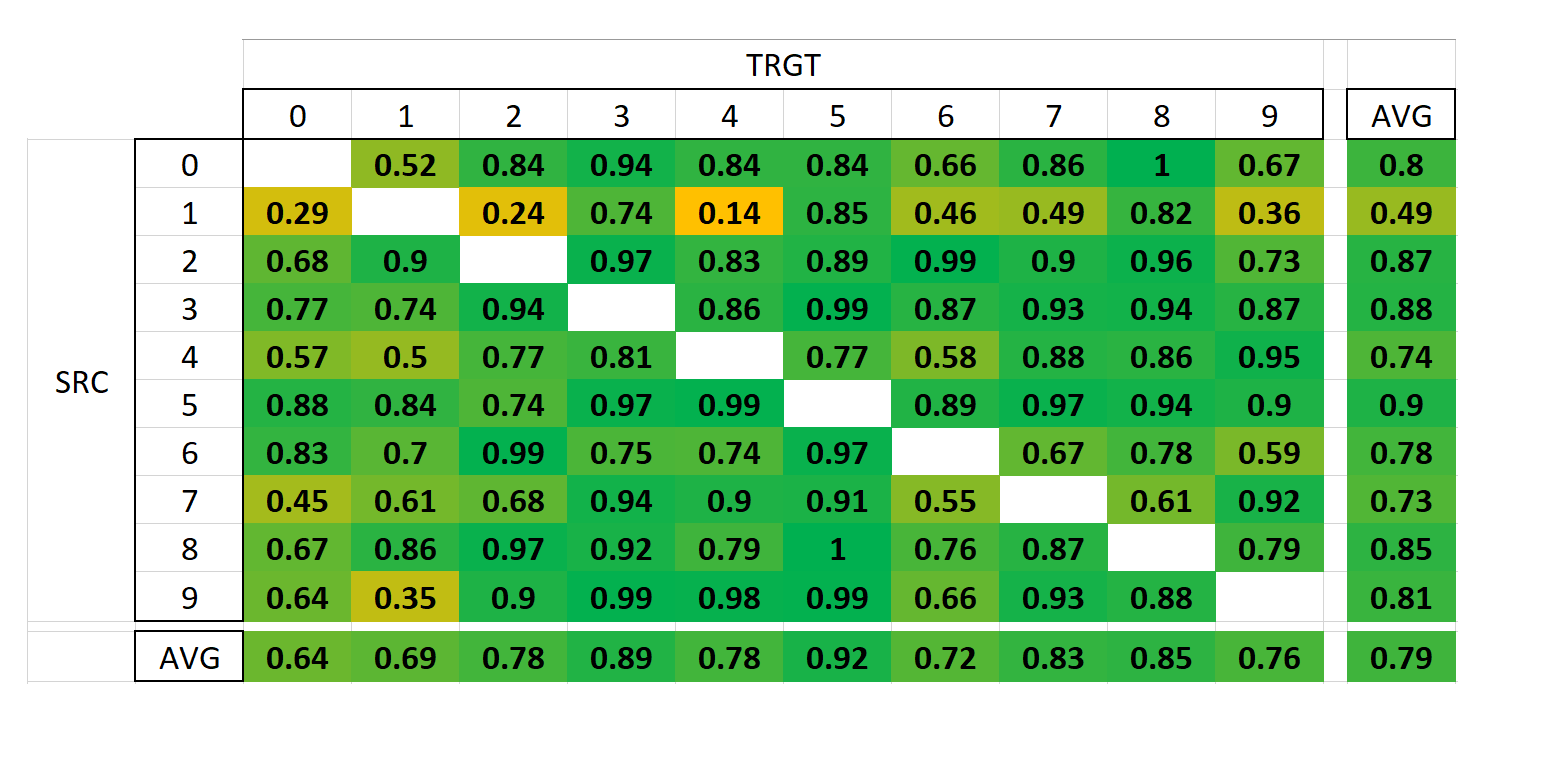}
  \vspace{-4pt}
  \caption{Success in misclassifying from a class in X axis to a class in Y axis. Darker background indicates a higher success rates.}
  \label{fig:mnist_heat}
\vspace{-8pt}
\end{figure}

\vspace{-8pt}
\section{Conclusion}\label{conclusion}

In conclusion, we have shown
that the clean-label targeted data omission attack -- TDO -- succeeds in attacking multiple, diverse learners, both in white-box and in black-box attack settings. We have empirically shown success across omission strategies at a low attack budget. We have additionally shown that TDO succeeds across datasets, despite significant differences among them in size, in number and type of features, and in the data itself. Further, we show that TDO succeeds in attacking multi-class classification cases. We also found out that the genetic method dominates the success in white-box attacks, and the KNN method dominates the success in black-box attacks. As desired, the effect of TDO on model accuracy is negligible, thus leaving the vast majority of samples in the attacked dataset intact.
Comparison of TDO black-box attacks performance to benchmark results \cite{schwarzschild2020just} demonstrates success rates comparable to, and even higher than, state-of-the-art results. 
In addition to a comprehensive empirical evaluation, we provide theoretical foundations for the data omission attack in the context of simplified PAC learning. 
Despite impressive attack capabilities and performance, TDO has open issues. Firstly, black-box success rates, although above the benchmark, need improvement. Additionally, omission may be detectable, e.g., when dataset size is known. Future research should focus on these open questions and on defense mechanisms against this attack.

\bibliographystyle{named}
\bibliography{main}

\end{document}


\maketitle
\appendix

\section{Theorem 3.1 -- proof details}
\label{appendix_A}
Theorem 3.1 has two parts. A sketch of their proof is presented in the main paper. The detailed proof of the two parts follows in this appendix.
\begin{proof}
(Theorem 3.1)
\begin{description}
\item [Part 1]

As $X-\widetilde{X}$ and $\widetilde{X}$ are a partition of $X$  we have that with probability at least $1-\delta$, 
$\epsilon + min_{h \in H} \mbox{\it Ls}(h) > \mbox{\it Ls}(g) = P'[g(x) \neq f(x)] =  P'[g(x) \neq f(x) \land x \notin \widetilde{X}]   +  P'[g(x) \neq f(x) \land x \in \widetilde{X}]$.  Thus, $P'[g(x) \neq f(x) \land x \in \widetilde{X}] < min_{h \in H} \mbox{\it Ls}(h) + \epsilon$. Note that $P'[g(x) \neq f(x) \land x \in \widetilde{X}] =   
P'[g(x) \neq f(x) \land x \in \widetilde{X} \land f(x) = 1] + P'[g(x) \neq f(x) \land x \in \widetilde{X} \land f(x) = -1]$.  As $P'[g(x) \neq f(x) \land x \in \widetilde{X} \land f(x) = -1] = 0$ by the definition of $P'$, $ P'[g(x) \neq f(x) \land x \in \widetilde{X} \land f(x) = 1] < min_{h \in H} \mbox{\it Ls}(h) + \epsilon$.  In addition, $P'[g(x) \neq f(x) \land x \in \widetilde{X} \land f(x) = 1] \ge P[g(x) \neq f(x) \land x \in \widetilde{X} \land f(x) = 1]$ as we did not decrease the probability on any subset of $\{x \in \widetilde{X} | f(x) = 1\}$ when constructing $P'$. Thus, $P[g(x) \neq f(x) \land x \in \widetilde{X} \land f(x) = 1]  < \epsilon + min_{h \in H} \mbox{\it Ls}(h)$ and $P[g(x) = f(x) \land x \in \widetilde{X} \land f(x) = 1] > \gamma - (\epsilon + min_{h \in H} \mbox{\it Ls}(h))$  as well as $P'[g(x) = f(x) \land x \in \widetilde{X} \land f(x) = 1] > \gamma -( \epsilon + min_{h \in H} \mbox{\it Ls}(h))$.   Thus, $P'[g(x) = 1 \land x \in \widetilde{X}] > \gamma - (\epsilon + min_{h \in H} \mbox{\it Ls}(h))$ as well as $P[g(x) = 1 \land x \in \widetilde{X}] > \gamma - (\epsilon + min_{h \in H} \mbox{\it Ls}(h))$ as desired.
\item [Part 2]

$X-\widetilde{X}$, $\widetilde{X}$ is a partition of $X$.  Thus,  with probability at least $1-\delta$, 
$\epsilon + min_{h \in H} \mbox{\it Ls}(h) > \mbox{\it Ls}(g) = P'[g(x) \neq f(x)] =  P'[g(x) \neq f(x) \land x \notin \widetilde{X}]   +  P'[g(x) \neq f(x) \land x \in \widetilde{X}] = P[g(x) \neq f(x) \land x \notin \widetilde{X}]   +  P'[g(x) \neq f(x) \land x \in \widetilde{X}]$.  

Further analyze $P'[g(x) \neq f(x) \land x \in \widetilde{X}] =   
P'[g(x) \neq f(x) \land x \in \widetilde{X} \land f(x) = 1] + P'[g(x) \neq f(x) \land x \in \widetilde{X} \land f(x) = -1] = P'[g(x) \neq f(x) \land x \in \widetilde{X} \land f(x) = 1] \ge P[g(x) \neq f(x) \land x \in \widetilde{X} \land f(x) = 1] $.  Thus, $\epsilon + min_{h \in H} \mbox{\it Ls}(h) \ge P[g(x) \neq f(x) \land x \notin \widetilde{X}] + P[g(x) \neq f(x) \land x \in \widetilde{X} \land f(x) = 1] $.  

As $P[x \in \widetilde{X} \land f(x) = -1 ] < \gamma^{'}$ we have that $P[g(x) \neq f(x)] = P[g(x) \neq f(x) \land x \notin \widetilde{X}] + P[g(x) \neq f(x) \land x \in \widetilde{X} \land f(x) = -1] + P[g(x) \neq f(x) \land x \in \widetilde{X} \land f(x) = 1] \leq \epsilon + min_{h \in H} \mbox{\it Ls}(h) +  P[g(x) \neq f(x) \land x \in \widetilde{X} \land f(x) = -1] \leq \epsilon + min_{h \in H} \mbox{\it Ls}(h) +  P[f(x) = -1 \land x \in \widetilde{X} ] \leq \epsilon + min_{h \in H} \mbox{\it Ls}(h) + \gamma^{'} $.  
\end{description}
\end{proof}

\section{Complete success rate tables}
Tables that present the leading success rates of TDO white-box and black-box attacks are in the main paper. In this appendix we provide tables of all of the results. 

Table \ref{tab:Result_WhiteBox_appendix} presents the success rates of TDO white-box attacks in experiment sets I-III. White-box attacks were performed against 6 victim learners with datasets IMDB, MNIST and {\em Synthetic}. This Table presents all attack results.

\begin{table}[h]
\small
\centering
\begin{tabular}{lllll}
\hline
Dataset & classes & victim & method & Results	\\
\hline
IMDB & 2 & 1DConvNet & Genetic & 0.80	\\
MNIST & 2 & ANN & Genetic & 1.00	\\
MNIST & 2 & GNB & Genetic & 1.00	\\
MNIST & 2 & GNB & Greedy & 1.00	\\
MNIST & 3 & ANN & Genetic & 1.00	\\
MNIST & 3 & SVM & Genetic & 1.00	\\
MNIST & 2 & KNN5 & Genetic & 0.90	\\
MNIST & 2 & SVM & Genetic & 0.82	\\
MNIST & 3 & KNN5 & Genetic & 0.55	\\
MNIST & 2 & ANN & Greedy & 0.54	\\
MNIST & 2 & KNN5 & Greedy & 0.25	\\
MNIST & 2 & SVM & Greedy & 0.05	\\
Synthetic & 2 & KNN5 & Genetic & 0.99	\\
Synthetic & 2 & ANN & Genetic & 0.88	\\
Synthetic & 2 & SVM & Genetic & 0.87	\\
Synthetic & 2 & ANN & Greedy & 0.86	\\
Synthetic & 2 & Dtree & Genetic & 0.85	\\
Synthetic & 2 & GNB & Genetic & 0.58	\\
Synthetic & 2 & Dtree & Greedy & 0.55	\\
Synthetic & 2 & GNB & Greedy & 0.52	\\
Synthetic & 2 & KNN5 & Greedy & 0.36	\\
Synthetic & 2 & SVM & Greedy & 0.17	\\
\hline
\end{tabular}
\normalsize
\caption{Success rates in white-box attacks (complete table)}
\label{tab:Result_WhiteBox_appendix}
\end{table}

Table \ref{tab:Result_BlackBox_appendix} presents the success rates of TDO black-box attacks in experiment sets I, II and IV. Black-box attacks were performed against 8 victim learners with datasets CIFAR-10, MNIST and {\em Synthetic}. This Table presents all attack results.

\begin{table}[h]
\small
\centering
\begin{tabular}{llllll}
\hline
Dataset & clss & Surrogate & victim & method & Results	\\
\hline
CIFAR10 & 10 & Googlenet & Resnet18 & KNN & 0.25	\\
CIFAR10 & 10 & Googlenet & MobleNetV2 & KNN & 0.15	\\
CIFAR10 & 10 & Googlenet & VGG11 & KNN & 0.14	\\
CIFAR10 & 10 & Googlenet & AlexNet & KNN & 0.08	\\
MNIST & 2 & X & KNN5 & KNN & 0.85	\\
MNIST & 2 & X & GNB & KNN & 0.80	\\
MNIST & 3 & X & ANN & KNN & 0.69	\\
MNIST & 2 & ANN & ANN & Greedy & 0.46	\\
MNIST & 2 & X & ANN & KNN & 0.45	\\
MNIST & 3 & GNB & Dtree & Genetic & 0.44	\\
MNIST & 2 & GNB & GNB & Genetic & 0.43	\\
MNIST & 2 & GNB & GNB & Greedy & 0.43	\\
MNIST & 2 & ANN & ANN & Genetic & 0.27	\\
MNIST & 2 & X & SVM & KNN & 0.18	\\
MNIST & 3 & X & Dtree & KNN & 0.16	\\
MNIST & 2 & SVM & SVM & Genetic & 0.15	\\
MNIST & 3 & X & KNN5 & KNN & 0.15	\\
MNIST & 3 & X & SVM & KNN & 0.10	\\
MNIST & 2 & SVM & SVM & Greedy & 0.07	\\
Synthetic & 2 & X & KNN5 & KNN & 1.00	\\
Synthetic & 2 & X & Dtree & KNN & 0.90	\\
Synthetic & 2 & X & ANN & KNN & 0.65	\\
Synthetic & 2 & X & SVM & KNN & 0.48	\\
Synthetic & 2 & X & GNB & KNN & 0.17	\\
\hline

\end{tabular}
\normalsize
\caption{Success rates in black-box attacks (complete table)}
\label{tab:Result_BlackBox_appendix}
\end{table}